\documentclass{article}

\usepackage{PRIMEarxiv}

\usepackage[utf8]{inputenc} 
\usepackage[T1]{fontenc}    
\usepackage{hyperref}       
\usepackage{url}            
\usepackage{booktabs}       
\usepackage{amsfonts}       
\usepackage{nicefrac}       
\usepackage{microtype}      
\usepackage{lipsum}
\usepackage{fancyhdr}       
\usepackage{graphicx}       
\graphicspath{{media/}}     
\usepackage{multirow}
\usepackage{subcaption}
\usepackage{natbib}
\usepackage{amsmath}
\usepackage{amssymb}
\usepackage{mathtools}
\usepackage{amsthm}
\DeclareMathOperator*{\minimize}{minimize}
\usepackage{bm}
\usepackage{comment}
\usepackage{subcaption}

\usepackage{xcolor}

\newtheorem{theorem}{Theorem}[section]

\newtheorem{lemma}[theorem]{Lemma}

\newcommand\blfootnote[1]{%
  \begingroup
  \renewcommand\thefootnote{}\footnote{#1}%
  \addtocounter{footnote}{-1}%
  \endgroup
}

\title{An Effective Flow-based Method for Positive-Unlabeled Learning: 2-HNC} 

\author{
  Dorit Hochbaum, Torpong Nitayanont \\
  Department of Industrial Engineering and Operations Research \\
  University of California, Berkeley \\
  Berkeley, CA\\
  \texttt{\{dhochbaum, torpong\_nitayanont\}@berkeley.edu} \\
}

\begin{document}
\maketitle

\begin{abstract}
In many scenarios of binary classification, only positive instances are provided in the training data, leaving the rest of the data unlabeled. 
This setup, known as positive-unlabeled (PU) learning, is addressed here with a network flow-based method which utilizes pairwise similarities between samples.
The method we propose here, 2-HNC, leverages Hochbaum’s Normalized Cut (HNC) and the set of solutions it provides by solving a parametric minimum cut problem.
The set of solutions, that are nested partitions of the  samples into two sets, correspond to varying tradeoff values between the two goals: high intra-similarity inside the sets and low inter-similarity between the two sets. 
This nested sequence is utilized here to deliver a ranking of unlabeled samples by their likelihood of being negative. 
Building on this insight, our method, 2-HNC, proceeds in two stages. 
The first stage generates this ranking without assuming any negative labels, using a problem formulation that is constrained only on positive labeled samples. 
The second stage augments the positive set with likely-negative samples and recomputes the classification. 
The final label prediction selects among all generated partitions
in both stages, the one that delivers a positive class proportion, closest to a prior estimate of this quantity, which is assumed to be given. 
Extensive experiments across synthetic and real datasets show that 2-HNC yields strong performance and often surpasses existing state-of-the-art algorithms.
\blfootnote{This work is an extended version of  \citep{nitayanontpositive}.}
\end{abstract}

\section{Introduction}
\label{sec:introduction}

When collecting data, there are scenarios where negative samples are difficult to verify or obtain, or when the absence of a positive label for a particular sample does not always imply that the sample is negative. 
This happens, for example, in the identification of malignant genes \citep{yang2012positive,yang2014ensemble} where a limited set of genes have been verified to cause diseases and labeled as \textit{positive} while many other genes have not been evaluated and remain \textit{unlabeled}. 
In personalized advertising \citep{yi2017scalable,bekker2020learning}, each advertisement that is clicked is recorded as a positive sample. However, an un-clicked advertisement is regarded as unlabeled as it could either be uninteresting to the viewer (\textit{negative}), or interesting but overlooked (\textit{positive}). Other domains with similar characteristics include fake reviews detection \citep{li2014spotting,ren2014positive} and remote sensing \citep{li2010positive}.
In such cases, the binary classification problem is formulated as a Positive-Unlabeled (PU) learning problem, where the given data consists of two sets: (i) the labeled set that contains only positive samples and (ii) the unlabeled set that contains samples from both classes without given labels. PU learning is a variant of binary classification that has been studied in various domains such as those that are mentioned above.

PU learning is related to a class of classification problems called one-class learning problem. PU learning and one-class learning are similar in the way that labeled samples that are provided in the problem instance only come from one class. However, unlabeled samples may or may not be given in one-class learning \citep{khan2014one}, whereas techniques in PU learning always utilize unlabeled samples. 
Another related class of problem is semi-supervised learning, where unlabeled samples are used in addition to the labeled set of samples from both classes, giving better performances than one-class learning methods \citep{lee2003learning,li2010positive}. PU learning is a special case of semi-supervised learning where no negative labeled samples are provided.

One way for dealing with the absence of negative labeled samples is to identify unlabeled samples that are likely negative. These likely-negative unlabeled samples are then 
added to the labeled set. A traditional classifier of choice is then trained using the augmented labeled set \citep{liu2002partially,li2003learning}. 
A related approach is to treat all unlabeled samples as noisy negative samples prior to applying a noise-robust classifier such as the noise-robust support vector machine method in \citep{liu2003building,claesen2015robust}.
Another common approach is to train a classifier on a modified risk estimator, in which each unlabeled sample can be regarded as positive and negative with different weights. This idea has been adopted in different learning methods such as neural network models \citep{du2014analysis,du2015convex,kiryo2017positive}, and random forest \citep{wilton2022positive} with a modified impurity function. A common prior information that many PU learning methods, including the works mentioned here, rely on is the fraction of positive samples in the data, which we denote as $\pi$.

The new method proposed here for PU learning is based on a network flow-based method called Hochbaum's Normalized Cut (HNC) \citep{hochbaum2009polynomial}. 
HNC is a clustering method that finds an optimal cluster of samples to achieve high intra-similarity within the cluster and low inter-similarity between the cluster and its complement, with a {\em tradeoff} parameter that balances the two goals. The problem was shown in \citep{hochbaum2009polynomial} to be solved, for all values of the tradeoff parameter, as a minimum cut problem on an associated graph. This method has been used in binary classification where both positive and negative labeled samples are available, e.g.\ \citep{yang2014supervised,baumann2019comparative}. Data samples are partitioned into a cluster and its complement set; the predicted label for an unlabeled sample is determined based on which set in the partition it belongs to. HNC is applicable in PU learning since it does not require labeled samples from both classes. Moreover, it makes use of the available unlabeled samples through their similarities with labeled samples and among themselves, making it advantageous when labeled data is limited.

As a transductive method, HNC partitions the given positive labeled and unlabeled samples, and predicts labels only for the given unlabeled samples.
This is different from inductive methods that make predictions for any unlabeled samples, whether they are the given unlabeled samples, or previously unseen
unlabeled samples. 
HNC can be extended to be used as an inductive classifier by assigning  an unseen unlabeled sample either the label of the subset in the partition to which it is closer, or the label that results in a smallest change in the objective value.

Our new method for PU learning, called \textit{2-HNC}, utilizes the unique features of HNC in two stages.
In the first stage, 2-HNC generates multiple partitions of data samples for different tradeoff values, efficiently, with a parametric cut procedure.
Based on the generated sequence of partitions,
we generate a ranking of the unlabeled samples in how likely they are to be negative. Using this ranking,
stage 2 determines a set of unlabeled samples that are likely negative and applies HNC using both the positive samples and the likely-negative unlabeled samples. Among all partitions generated in both stages, the one whose fraction of positive samples is closest to the provided prior $\pi$ is selected. The prediction for the unlabeled samples is based on this selected partition. 

We provide here an experimental study, using both synthetic and real data, comparing the performance of 2-HNC with that of leading methods of PU learning.
These include two commonly used benchmarks in PU learning works, \textit{uPU} \citep{du2014analysis,du2015convex} and \textit{nnPU} \citep{kiryo2017positive}, as well as a recent state-of-the-art tree-based method, \textit{PU ET} \citep{wilton2022positive}.  Our experiments  demonstrate that 2-HNC outperforms these methods for PU learning tasks.  

The major contribution here is the introduction of the 2-HNC method which is particularly suitable in situations where labeled samples are provided only for one class, such as PU learning.
The second contribution is the method of ranking the likelihood of samples to be of certain label based on the order in which it joins one set in the sequence of partitions. 
This has the potential of applicability in setups beyond PU learning, such as in active learning, where the task is to select the most informative samples for labeling. Our ranking method could be used to infer the informativeness of an unlabeled sample based on how early or late it joins the set in the partitions. Another potential use is in outlier detection. The order in which unlabeled samples move from one set to the other set in the partition helps determine the degree to which samples are dissimilar from the set of known regular samples or the majority set. 
The third contribution of this work is proving that clustering in which the goal is to combine the maximization of the intra-similarities in {\em both} the cluster and its complement with the minimization of the inter-similarity between the two sets, is equivalent to a form of HNC.  We call that, the {\em Double Intra-similarity Theorem}, Theorem \ref{thm:generalized}.
The implication of this theorem is that including both intra-similarities in the objective function is a clustering problem that is solvable as a minimum cut on an associated graph.

\section{Related Research}\label{sec:related-works}

The main challenge of PU learning is the lack of negative labeled samples. A number of methods utilize a preprocessing step to identify a set of unlabeled samples that are likely to be negative and include them as negative labeled samples prior to training a traditional binary classifier. 
One such method is the Spy technique \citep{liu2002partially} that, first, selects a few positive labeled samples as \textit{spies}. These positive spies, together with all unlabeled samples, are then used as negative labeled samples to train a binary classifier, alongside the given positive labeled set. With a binary classifier trained on this data, unlabeled samples with lower posterior probability 
than the spies are considered likely to be negative. 
The Rocchio method of \citet{li2003learning} marks unlabeled samples that are closer to the centroid of unlabeled samples than that of positive labeled samples as likely negative. The work of \citet{lu2010semi} used Rocchio to expand the positive labeled set when a small positive labeled set is given 

Instead of selecting some unlabeled samples as negative samples, methods such as \citep{liu2003building,claesen2015robust} regard all unlabeled samples as noisy negative samples. By doing so, the noise-robust support vector machine in \citep{liu2003building} penalizes misclassified positive samples more than misclassified negative samples, which were originally unlabeled, as their pseudo-negative labels are less reliable than the known labels of the given positive samples.
The method in \citep{claesen2015robust} is a bagging method in which multiple SVM models of \citep{liu2003building} are trained on bootstrap resamples 
of the data. 

Another common approach in recent works is to train a model based on an empirical risk estimator, modified in the context of PU learning. In the \textit{uPU} method of \citet{du2014analysis,du2015convex}, the authors proposed an unbiased risk estimator for PU data to use as a loss function in a neural network model of choice. The loss due to each unlabeled sample is a weighted combination of the loss assuming that the sample is positive and the loss assuming that the sample is negative. These weights are dependent on the prior $\pi$, which is the fraction of positive samples in the data. 
\citet{kiryo2017positive} mitigates the overfitting nature of uPU via a non-negative risk estimator in their state-of-the-art method known as \textit{nnPU}. Research on other classifiers, besides deep learning models, that apply a similar idea include a random forest model called \textit{PU ET}, by \citet{wilton2022positive}, in which the impurity function is modified for PU data by considering each unlabeled sample as a weighted combination of positive and negative samples.  PU ET gives competitive results, especially on tabular data type where deep learning PU methods are not always effective.

There are methods, other than the ones listed above, which rely on pairwise similarities between samples. In the label propagation method of \citep{carnevali2021graph}, a graph representation of the data is constructed with edge weights that reflect pairwise similarities. The likelihood of being negative for each unlabeled sample is inferred based on its shortest path distance on the graph to the positive labeled set. Labels are then propagated from the positive and likely-negative unlabeled samples to the remaining unlabeled ones. \citet{zhang2019positive} presented a maximum margin-based method that penalizes similar samples that are classified differently.
While methods such as those presented in \citep{carnevali2021graph,zhang2019positive} employed graph representation of the data as well as pairwise similarities, a network-flow based approach, as the one presented here, has never been utilized in PU learning.

Hochbaum's Normalized Cut or HNC \citep{hochbaum2009polynomial} has been used in binary classification, where labeled samples from both classes are given, and exhibited competitive performance \citep{baumann2019comparative,spaen2019hnccorr,yang2014supervised}. Here, we devise a variant of HNC for PU learning, called \textit{2-HNC}, which is compared to the following benchmarks: \textit{uPU} \citep{du2014analysis}, \textit{nnPU} \citep{kiryo2017positive} and \textit{PU ET} \citep{wilton2022positive}. uPU and nnPU are selected because they are broadly considered as standard PU learning benchmarks. Between the two, nnPU exhibited competitive performance more consistently \citep{kiryo2017positive}. 
PU ET, a recent state-of-the-art method, exhibited leading performance, particularly on tabular data where it outperformed deep learning models \citep{wilton2022positive}. Similar to most PU methods, 2-HNC as well as the benchmark methods in the experiments presented here - uPU, nnPU and PU ET - rely on the given prior $\pi$, which is the fraction of positive samples in the data.

\section{Preliminaries and Notations} \label{sec:prelim}

\subsection{Problem Statement}

We are given a dataset $V$ that consists of a set of positive labeled samples $L^+$ and a set of unlabeled samples $U$. 
Each sample $i \in L^+$ has an assigned label $y_i = +1$, indicating the positive class. 
Each unlabeled sample $j \in U$ belongs to either the positive or the negative class. Their labels, which are considered the {\em ground truth}, are unavailable to the prediction method. 
Each sample $i \in L \cup U$, is associated with a feature vector representation $\bm{x}_i \in \mathbb{R}^H$ where $H$ the number of features for each sample. 
 The goal is to predict the label of the unlabeled samples in $U$.

\subsection{Graph Notations} \label{subsec:notation}
The PU learning task is formalized here as a graph problem. To that end, we introduce relevant concepts and notations.

For an undirected graph $G=(V,E)$, where $V$ is the set of nodes (or vertices) and each edge $[i,j] \in E$ has a weight of $w_{ij}$, we define the function $C(V_1,V_2)$ for two sets $V_1,V_2 \subseteq V$:  $C(V_1,V_2) := \sum_{[i,j] \in E, i \in V_1, j \in V_2} w_{ij}$. That is, this is the sum of weights of edges that have one endpoint in $V_1$ and the other in $V_2$. 

For a directed graph $G=(V,A)$, where each arc $(i,j) \in A$ has a weight of $w_{ij}$, $C(V_1,V_2)$ is defined as $\sum_{(i,j) \in A, i \in V_1, j \in V_2} w_{ij}$, which is the sum of weights of the arcs that go from nodes in $V_1$ to those in $V_2$. 
We denote the \textit{weighted degree} of node $i \in V$ by $d_i$ where $d_i = \sum_{[i,j] \in E, j \in V} w_{ij}$. The \textit{volume} of the set $S \subseteq V$ is denoted by $d(S)=\sum_{i \in S} d_i$.

The \textit{minimum $(s,t)$-cut problem} is defined on a directed graph $G=(V,A)$ in which two nodes in $V$ are specified to be the source node, $s$, and the sink node, $t$. We call the directed graph with two designated nodes $s$ and $t$ an $(s,t)$-graph. 
An $(s,t)$-cut is a partition of the set of nodes $V$ into $(S,T)$ such that $s \in S, t \in T, S \subseteq V$ and $T = V \backslash S$. The set $S$ in the partition is referred to as the \textit{source set} whereas the set $T$ is the \textit{sink set}.
The \textit{capacity} of the $(s,t)$-cut $(S,T)$ is $C(S,T)$. 
A \textit{minimum $(s,t)$-cut} is the $(s,t)$-cut that has the minimum capacity among all cuts. For the remainder of this paper, we will refer to the minimum $(s,t)$-cut simply as \textit{the minimum cut}.

In the context of the minimum cut problem, as well as the related \textit{maximum flow problem}, the weights of the arcs in an $(s,t)$-graph are called \textit{capacities}. This will also be the term that we use in this work to refer to the weights of arcs of an $(s,t)$-graph.

\subsection{Nested Cut Property and Parametric Cut} \label{subsec:nestedcut}
Consider an $(s,t)$-graph $G_{st} = (V,A)$ whose capacities of the source and sink adjacent arcs, $A_s := \{(s,i) \in A \;|\; i \in V\}$ and $A_t := \{(i,t) \in A \;|\; i \in V\}$, are functions of a parameter $\lambda$.
$G_{st}(\lambda)$ is said to be a \textit{parametric flow graph} if the capacities of arcs in $A_s$ and $A_t$ are non-decreasing and non-increasing (or non-increasing and non-decreasing) functions of $\lambda$, respectively.

For a list of increasing values of $\lambda$: $\lambda _1 < \lambda _2\ldots < \lambda _q$, let the corresponding minimum cut partitions be $(S_1, \overline{S}_1),$ $(S_2, \overline{S}_2), \dots,$ $(S_q, \overline{S}_q)$. 
These partitions satisfy the nested cut property:

\begin{lemma}[Nested Cut Property] \label{lemma:nestedcut}\citep{gallo1989fast,hochbaum1998pseudoflow,hochbaum2008pseudoflow} 
Given a parametric flow graph $G_{st}(\lambda)$ with source and sink adjacent arc capacities that are non-increasing and non-decreasing functions of a parameter $\lambda$, respectively, and a sequence of parameter values $\lambda _1 < \lambda _2\ldots < \lambda _q$.
Then, for the corresponding minimum cut partitions, $(S_1, \overline{S}_1),$ $(S_2, \overline{S}_2), \dots,$ $(S_q, \overline{S}_q)$, the sink sets are nested:
$\overline{S}_1 \subseteq \overline{S}_2 \subseteq \dots \subseteq \overline{S}_q$. On the other hand, if the source and sink adjacent arcs of the parametric flow graph have non-decreasing and non-increasing capacities, the source sets are nested instead: $S_1 \subseteq S_2 \subseteq \dots \subseteq S_q$.
\end{lemma}

\begin{figure}[h!] 
\centering
  \includegraphics[width=.98\linewidth]{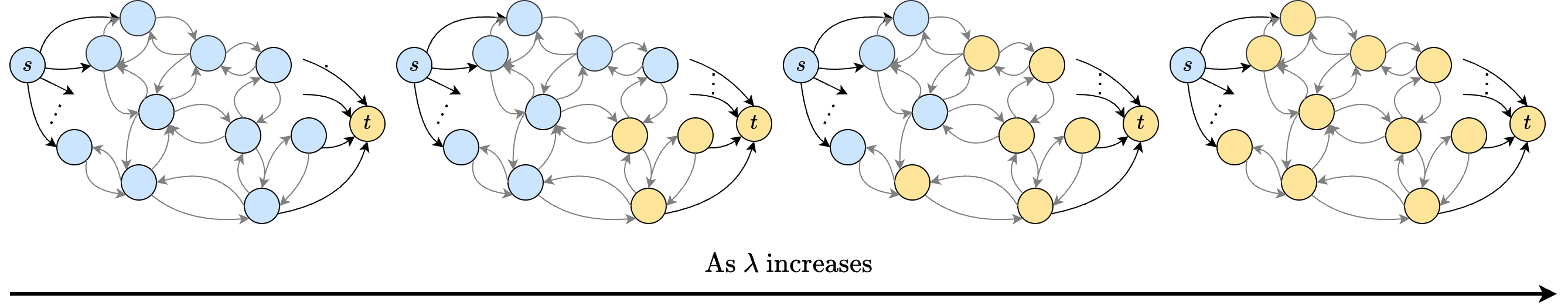}
  \caption{Nested cut partitions of a parametric flow graph in which the capacities of source and sink adjacent arcs are non-increasing and non-decreasing functions of the parameter $\lambda$. Sink sets (in yellow) of smaller $\lambda$ are nested in those of larger $\lambda$.} 
  \label{fig:nestedsequence}
\end{figure}


The nestedness of the partitions sequence is illustrated in Figure \ref{fig:nestedsequence} for a parametric flow graph that has non-increasing and non-decreasing source adjacent and sink adjacent arc capacities, respectively. For a sufficiently small value of $\lambda$, the sink set consists only of the sink node $t$ and nodes that are connected to $t$ with infinite capacities, if any. The remaining nodes, and the source node $s$, are in the source set (in blue). As the value of $\lambda$ increases, at certain values, the partition changes and some nodes from the source set are added to the sink set. Once a node moves to the sink set for a particular value of $\lambda$, it remains in the sink set for all larger $\lambda$. 

The \textit{parametric (minimum) cut} procedure finds the minimum cuts and respective partitions of the parametric flow graph $G_{st}(\lambda)$ for all possible $\lambda$ values. 
\citet{gallo1989fast} introduced the parametric cut procedure based on the push-relabel algorithm and \citet{hochbaum1998pseudoflow, hochbaum2008pseudoflow} introduced a parametric cut procedure based on HPF (Hochbaum's PseudoFlow). The \textit{fully parametric} version of both procedures identify all the breakpoints and their respective minimum cut partitions for all values of $\lambda$, whereas the \textit{simple parametric} version finds minimum cuts for a given list of increasing values of $\lambda$. The fully parametric and simple parametric versions of both procedures run in the complexity of a single minimum cut procedure.
While the theoretical complexities of both procedures by \citet{gallo1989fast} and \citet{hochbaum1998pseudoflow, hochbaum2008pseudoflow} are the same, the one of \citet{gallo1989fast} has not been implemented, whereas the fully parametric HPF is implemented and available at \citep{fullpara}, and the simple parametric HPF is available at \citep{simplepara}. 

Note that in practice, simple parametric pseudoflow runs faster than fully parametric pseudoflow. 
Therefore, in this work, when we solve a parametric cut problem, we solve it for a list of parameter values using the simple parametric version \citep{simplepara}.

\subsection{Hochbaum's Normalized Cut (HNC)} \label{subsec:HNC}
We represent a dataset as 
an undirected graph $G=(V,E)$, where each node in the set $V$ corresponds to a sample in the data. 
The notation $V$ refers henceforth to both the dataset and the set of nodes in the graph.
An edge $[i,j] \in E$ connects a pair of nodes (or samples), carrying the weight $w_{ij}$ that quantifies the \textit{pairwise similarity} between samples $i$ and $j$.

For a partition of the set of samples $V$ into two disjoint sets $S$ and $\overline{S}$, where $\overline{S} = V \backslash S$, $C(S, \overline{S}) = \sum_{[i,j]\in E, i\in S, j\in \overline{S}} w_{ij}$ represents the \textit{inter-similarity} between samples in $S$ and its complement $\overline{S}$.
$C(S, S) = \sum_{[i,j]\in E, i\in S, j\in S, i < j} w_{ij}$ denotes the total sum of similarities between samples within the set $S$, the \textit{intra-similarity} of $S$. 

With the set of samples $V$ and pairwise similarities $w_{ij}$ for $i, j \in V$, 
the goal of HNC is to find a partition of $V$ to two non-empty sets: the cluster $S$ and its complement $\overline{S}$ 
that optimizes the tradeoff between two objectives: 
\textit{high intra-similarity} within the set $S$ and \textit{small inter-similarity} between $S$ and $\overline{S}$. 
The problem formulation of HNC, for a tradeoff parameter $\mu \geq 0$, is then given as follows:
\begin{equation} \label{eq:HNC-lambda}
    (\text{HNC+}) \quad \minimize_{\varnothing \subset S \subset V} \quad C(S,\overline{S}) - \mu \; C(S,S)
\end{equation} 

Because of the symmetry between $S$ and $\overline{S}$, the problem can be alternatively presented for the tradeoff between the intra-similarity within $\overline{S}$ and the inter-similarity between it and its complement.
\begin{equation} \label{eq:HNC-lambda-2}
(\text{HNC-}) \quad \minimize_{\varnothing \subset S \subset V} \quad C(S,\overline{S}) - \mu \; C(\overline{S},\overline{S})
\end{equation} 

One might consider a variant of HNC that incorporates both intra-similarities, $C(S,S)$ and $C(\overline{S}, \overline{S})$, and as such generalizing both HNC+ (\ref{eq:HNC-lambda}) and HNC- (\ref{eq:HNC-lambda-2}). This variant, with two tradeoff weights $\alpha \geq 0$ and $\beta \geq 0$, is given as problem (HNC$\pm$): 
\begin{equation} \label{eq:linearcomb}
  (\text{HNC$\pm$}) \quad     \minimize_{\varnothing \subset S \subset V} \quad C(S,\overline{S}) - \alpha \; C(S,S) -\beta \; C(\overline{S}, \overline{S}) 
\end{equation} 
As proved in the next theorem, HNC$\pm$ (\ref{eq:linearcomb}) 
is equivalent to either HNC+  or HNC-, depending on the relative values of $\alpha $ and $ \beta$.

\begin{theorem}[{\bf Double Intra-similarity Theorem}] \label{thm:generalized}
HNC$\pm$ (\ref{eq:linearcomb}) is equivalent to HNC+ (\ref{eq:HNC-lambda}) when $\alpha \geq \beta$ for $\mu=\frac{\alpha-\beta}{1+\beta}$, and is equivalent to HNC- (\ref{eq:HNC-lambda-2}) when $\alpha < \beta$ for $\mu=\frac{\beta-\alpha}{1+\alpha}$.
\end{theorem}
\begin{proof}
$C(V,V)$ is a constant, which we denote by $W_V$, and is equal to $C(S,\overline{S})+C(S,S)+C(\overline{S},\overline{S})$ for any nonempty $S \subset V$. Hence, the objective function  of (\ref{eq:linearcomb}) can be written as $C(S,\overline{S}) - \alpha \; C(S,S) -\beta (W_V - C(S,\overline{S}) - C(S,S)) = (1+\beta) (C(S,\overline{S}) - \frac{\alpha-\beta}{1+\beta} C(S,S)) - \beta W_V$.
Minimizing this function is equivalent to solving (\ref{eq:HNC-lambda}) with the tradeoff $\mu=\frac{\alpha-\beta}{1+\beta} \geq 0$ when $\alpha \geq \beta$.

Alternatively, the objective function of (\ref{eq:linearcomb}) can be written as $C(S,\overline{S}) - \alpha (W_V - C(S,\overline{S}) - C(\overline{S},\overline{S})) -\beta \; C(\overline{S},\overline{S}) = (1+\alpha) (C(S,\overline{S}) - \frac{\beta-\alpha}{1+\alpha} C(S,S)) - \alpha W_V$. Hence, minimizing this objective is equivalent to solving (\ref{eq:HNC-lambda-2}) with $\mu=\frac{\beta-\alpha}{1+\alpha} \geq 0$ when $\alpha < \beta$.
\end{proof}

Therefore, instead of solving (\ref{eq:linearcomb}) where the two intra-similarities are shown explicitly, it is sufficient to consider either HNC+  or HNC- depending on whether we put more weight on the intra-similarity of $S$, or of $\overline{S}$.
We note that in prior applications of HNC to binary classification, e.g.\ 
\citep{yang2014supervised,baumann2019comparative}, the model was the one that considered the intra-similarity in $S$ only, as in HNC+.

Applying HNC in binary classification, when labeled samples from both classes are given, the goal is to partition a data that consists of the positive and negative labeled sets, $L^+$ and $L^-$, as well as the unlabeled set, $U$, into $S$ and $\overline{S}$, and predict the labels of unlabeled samples in $U$ accordingly. In previous works, e.g.\ \citep{yang2014supervised,baumann2019comparative}, the labeled sets are used as \textit{seeds} and HNC+ is solved with the restriction that $L^+ \subseteq S \subseteq V \setminus L^-$. Unlabeled samples in the optimal $S^*$ and $\overline{S}^*$ are predicted positive and negative, respectively. 

\section{2-HNC: The two-stage method for PU learning}\label{sec:HNC}
In this section, we describe the \textit{2-HNC} method where HNC is applied in two stages in PU learning where only the positive labeled set $L^+$ and the unlabeled set $U$ are given. We then show how the optimization problems in 2-HNC are solved as parametric cut problems on associated graphs
in the complexity of a single minimum cut procedure.

\subsection{2-HNC for PU Learning} \label{subsec:2HNC}

The \textit{2-HNC} method consists of two stages. In stage 1, we solve HNC- using only the labels of the given positive labeled samples. In stage 2, we utilize the likely-negative samples extracted from the unlabeled set based on the result of the first stage, in solving HNC+ using both the positive labeled set and the likely-negative set. The output solution is the one data partition, among those that were generated in both stages, that has the fraction of positive samples closest to the ratio $\pi$, given as prior.

\subsubsection{Stage 1: Solving HNC-  with Positive Labeled Samples} \label{subsec:2HNC-Stage1}

The given positive labeled set $L^{+}$ is used as the seed set for the set $S$ in HNC+ and HNC-. Since no negative labeled samples are provided, $L^{-}=\varnothing$, meaning that no seed sample is required to be in $\overline{S}$. The seed set constraint imposed on HNC+ and HNC- is then $L^+ \subseteq S$. 

Without a seed set for $\overline{S}$, HNC+ is not well defined: the optimal solution to HNC+ is always $(S^*, \overline{S}^*) = (V, \varnothing)$ for any tradeoff $\mu \geq 0$. That is, HNC+ has only the trivial solution in which all unlabeled samples are predicted to be positive.
On the other hand, HNC- , with only positive labeled samples, gives non-trivial partitions for various values of the tradeoff parameter.

The optimal data partition for HNC- is dependent on the value of the tradeoff parameter $\mu$. We solve HNC- under the constraint $L^+ \subseteq S$, for a list of $q$ increasing nonnegative values of the tradeoff parameter $\mu$ as a parametric cut problem on an associated parametric flow graph, which satisfies the nested cut property in Lemma \ref{lemma:nestedcut}.  
For $\mu=0$, the optimal partition $(S^*, \overline{S}^*)$ is $(V, \varnothing)$. As $\mu$ increases, the optimal partition changes according to the nested cut property,
 until $\mu$ reaches a sufficiently large value, at which $(S^*, \overline{S}^*)$ is $(L^{+}, V \backslash L^{+})$.
The result of the associated parametric cut problem is a sequence of data partitions: $(S^*_1, \overline{S}^*_1), (S^*_2, \overline{S}^*_2), \dots, (S^*_q, \overline{S}^*_q)$, that correspond to increasing values of $\mu$, such that $\overline{S}^*_1 \subseteq \overline{S}^*_2 \subseteq \dots \subseteq \overline{S}^*_q$. 
We discuss in Section \ref{subsec:parametric} the procedure of solving HNC- as a parametric cut problem on an associated graph. The graph construction is also explained there and shown to be a parametric flow graph, with respect to the parameter $\mu$. Stage 1 ends here with the nested data partition sequence, that is the optimal solution to HNC- for different tradeoff values, as an output. 

\subsubsection{Stage 2: Solving HNC+ with Positive Labeled Samples and Likely-Negative Unlabeled Samples} \label{subsec:2HNC-Stage2}
Solving HNC- in stage 1 does not require negative labeled samples. For a given list of $q$ values of the tradeoff parameter $\mu$, stage 1 provides a partition of data samples into the positive negative prediction sets for each value.
Note that HNC- only considers the scenario where the intra-similarity of the negative prediction set $\overline{S}$ is given higher importance than that of the positive prediction set $S$. 
In stage 2, we consider HNC+, which is
ill-defined in the absence of a negative labeled set, as discussed in Section \ref{subsec:2HNC-Stage1}. To address that, we add to the problem the seeds for $\overline{S}$ selected from a set of unlabeled samples that are likely to be negative, $L^{N}$.
The procedure to identify $L^N$, called \textit{SelectNeg}, is based on the results of stage 1 and is described next.

SelectNeg takes as input the sequence of optimal data partitions $(S^*_1, \overline{S}^*_1), (S^*_2, \overline{S}^*_2), \dots, (S^*_q, \overline{S}^*_q)$, which are the results of solving HNC- for a list of nonnegative and increasing values of $\mu$. The nested sequence $\overline{S}^*_1 \subseteq \overline{S}^*_2 \subseteq \dots \subseteq \overline{S}^*_q$ starts from $\overline{S}^*_1=\varnothing$ and expands until $\overline{S}^*_q = V \backslash L^{+}$, which is the largest possible since we require $L^+$ to be in $S_q^*$. 

The implication of the nestedness is that, for an unlabeled sample that is predicted negative for a particular $\mu$, it is also predicted negative for any larger value of $\mu$. 
In other words, the partition sequence provides a ranking of unlabeled samples in terms of their likelihood to be predicted negative according to the order in which they join the sink set.  The lower the value of the parameter, or the earlier they join, the higher is the likelihood of being of negative label.
In Figure \ref{fig:nestedsequence}, consider the sample (or the node) that is in the negative prediction set or the sink set (in yellow) in the second graph, which corresponds to the second value of $\mu$.
The node remains in the sink set for all subsequent partitions, for larger values of $\mu$.

We use the information from the partitions produced in stage 1 for ranking the unlabeled samples according to their likelihood of being negative.  In particular, we rely on the order in which a sample joins the negative set of the partitions from stage 1. For example, consider two unlabeled samples $i, j \in U$. Let the indices of the first partition such that sample $i$ is put in the negative prediction set be $q_i := \min\{\eta \; | \; i \in \overline{S}^{*}_\eta\}$, and the index for $j$, $q_j$, is defined in a similar way. 
Without loss of generality, assume that $q_i < q_j$. According to the nested cut property, we have both of $i$ and $j$ in the positive prediction set for $\mu_1, \mu_2, \dots, \mu_{q_i-1}:$ $S^*_{1}, S^*_2,\dots, S^{*}_{q_i-1}$ as well as the negative prediction set for $\mu_{q_j}, \mu_{q_j+1}, \dots, \mu_{q}:$ $\overline{S}^*_{q_j}, \dots, \overline{S}^{*}_{q}$. The difference between the two samples is in the partitions corresponding to $\mu_{q_i}, \mu_{q_i+1}, \dots, \mu_{q_j-1}$, where sample $i$ is in the negative prediction sets $\overline{S}^*_{q_i}, \dots, \overline{S}^{*}_{q_j-1}$ and sample $j$ is in the positive prediction sets $S^*_{q_i}, \dots, S^{*}_{q_j-1}$. This implies that sample $i$ is more dissimilar from the positive set (more similar to the other set, which is the negative set) than sample $j$. Hence, for two unlabeled samples $i, j \in U$ where $q_i < q_j$, sample $i$ is more likely to be negative than sample $j$.

In SelectNeg, the number of unlabeled samples to be selected as likely-negative samples, or the size of the set $L^N$, is chosen in this work to be equal to $(\frac{1-\pi}{\pi})|L^+|$, so that the proportion of likely-negative samples among the labeled data, or $L^+ \cup L^N$, is equal to the proportion of negative samples in the entire data, which is $1-\pi$. Suppose we call $q_i$ \textit{the first-sink-set index of sample $i$}. Then, to select the likely-negative samples, we select $|L^N| = (\frac{1-\pi}{\pi})|L^+|$ unlabeled samples with the smallest first-sink-set indices.

Once the likely-negative set $L^N$ is formed, we use the positive labeled set $L^{+}$ and the likely-negative set $L^N$ as the seed sets for $S$ and $\overline{S}$, respectively, and solve HNC+ with the seed set constraint $L^+ \subseteq S \subseteq V \backslash L^N$. As a result, the output from stage 2 is another sequence of data partitions, which are the optimal solutions to HNC+ for the list of values of the tradeoff parameter $\mu$.

\subsubsection{Combining Results From Both Stages}

Among all data partitions generated in both stages, we select the partition whose positive fraction, computed as $\frac{|S^*|}{|V|}$ for a partition $(S^*, \overline{S}^*)$, is closest to the prior $\pi$. Unlabeled samples in $S^*$ of the selected partition are predicted positive, and those in $\overline{S}^*$ negative.

\subsection{Solving Parametric Cut Problems in 2-HNC} \label{subsec:parametric}

We mentioned in Section \ref{subsec:2HNC} that 2-HNC involves solving HNC+ and HNC- as parametric cut problems on associated graphs. We first explain how the two problems are solved for a single tradeoff $\mu \geq 0$ as minimum cut problems, in Section \ref{subsec:hnc-mincut}. In 2-HNC, we solve them for all values of $\mu$ in a given list, prior to selecting one partition from all that are generated. We describe how this is done as parametric cut problems in Section \ref{subsec:hnc-param}. 

\subsubsection{Solving HNC+ and HNC- for a Tradeoff Parameter \texorpdfstring{$\mu$}{$\mu$} as a Minimum Cut Problem} \label{subsec:hnc-mincut}
HNC+ and HNC- are special cases of {\em monotone integer programs}, \citep{hochbaum2002IP3,Hoc21-IP2IP3}, and as such can be solved as a minimum cut problem on an associated graph, which is a mapping from the integer programming formulation of the problems \citep{hochbaum2009polynomial}.

Using the standard formulations of HNC+ and HNC-, in the associated graph, there is a node for each sample, and a node for each pair of samples. As a result, the size of this graph is quadratic in the size of the data.  However, \citet{hochbaum2009polynomial} proved that there is a compact equivalent formulation for HNC in that the associated graph has number of nodes equal to the number of samples, $|V|$, only. These alternative formulations are shown for HNC+ and HNC- in the following lemma.

\begin{lemma}
    HNC+ is equivalent to the following problem:
\begin{equation} \label{eq:HNC-lambda-final}
\minimize_{\varnothing \subset S \subset V} \; C(S,\overline{S}) - \; \lambda \; \sum_{i \in S} \; d_i
\end{equation}
and HNC- is equivalent to
\begin{equation}
\label{eq:HNC-lambda-final-2}
\minimize_{\varnothing \subset S \subset V} \; C(S,\overline{S}) - \; \lambda \; \sum_{i \in \overline{S}} \; d_i 
\end{equation}
where $\lambda = \frac{\mu}{\mu+2}$ and $d_i = \sum_{[i,j] \in E, j \in V\backslash\{i\}} w_{ij}$ for $i \in V$.
\end{lemma}

\begin{proof}
$C(S,S) = \sum_{[i,j] \in E, i,j \in S, i < j} w_{ij} = \frac{1}{2} \sum_{i \in S} \sum_{[i,j] \in E, j \in S\backslash \{i\}} w_{ij}$ since $w_{ij}=w_{ji}$. 
$\sum_{i \in S} \sum_{[i,j] \in E, j \in S \backslash \{i\}} w_{ij} = \sum_{i \in S} (\sum_{[i,j] \in E, j \in V \backslash \{i\}} w_{ij} - \sum_{[i,j] \in E, j \in \overline{S}} w_{ij}) = \sum_{i \in S} d_i - C(S, \overline{S})$. Hence, $C(S,S) = \frac{1}{2}(\sum_{i \in S} d_i - C(S, \overline{S}))$.

We rewrite the objective of HNC+ as $C(S,\overline{S}) - \frac{\mu}{2}(\sum_{i \in S} d_i - C(S, \overline{S})) = (1+\frac{\mu}{2}) (C(S,\overline{S}) - \frac{\mu}{\mu+2} \sum_{i \in S} d_i)$. 
Hence, HNC+ can be solved by minimizing (\ref{eq:HNC-lambda-final}): $C(S,\overline{S}) - \lambda \sum_{i \in S} d_i$, with $\lambda = \frac{\mu}{\mu+2}$.
The equivalence of HNC- and (\ref{eq:HNC-lambda-final-2}) can be shown similarly by rewriting $C(\overline{S},\overline{S})$ in HNC- as $\frac{1}{2}(\sum_{i \in \overline{S}} d_i - C(S, \overline{S}))$.
\end{proof}


In a binary classification problem where both labeled sets $L^+$ and $L^-$ are given, the seed set constraint is $L^+ \subseteq S \subseteq V \backslash L^-$. The remaining samples in $V \backslash (L^+ \cup L^-)$ make up the unlabeled set $U$. Under this constraint, the solution to HNC+ for a tradeoff parameter $\mu$, which is solved as (\ref{eq:HNC-lambda-final}) with a tradeoff parameter $\lambda = \frac{\mu}{\mu+2}$, is obtained from the minimum cut solution of the associated $(s,t)$-graph, $G^{+}_{st}(\lambda)$. The use of the superscript ``$+$'' for $G^{+}_{st}(\lambda)$ is to reflect the emphasis on the intra-similarity of $S$, or the positive prediction set, in HNC+. 

The graph $G^{+}_{st}(\lambda)$ associated with (\ref{eq:HNC-lambda-final}) with the constraint $L^+ \subseteq S \subseteq V \backslash L^-$ is illustrated in Figure \ref{fig:HNC-1}. 
The graph is constructed as follows. 
The nodes in the graph consist of the set of nodes that represent samples in $V$, and two appended nodes: the source node $s$ and the sink node $t$. For each pair of samples $i,j \in V$, there are arcs $(i,j)$ and $(j,i)$, whose capacities are equal to the pairwise similarity between $i$ and $j$, $w_{ij}$. The set of source adjacent arcs is $A_s := \{(s,i) | i \in L^+ \cup U\}$ where the capacities of $(s,i)$ for $i \in L^+$ is $\infty$, and for $i \in U$, is $\lambda d_i$. The set of sink adjacent arcs is $A_t := \{(i,t) | i \in L^{-}\}$. Each $(i,t)$ has a capacity of $\infty$.


It was shown by \citet{hochbaum2009polynomial} that, if $(\{s\} \cup S^*, \{t\} \cup \overline{S}^*)$ is the minimum cut solution of $G^{+}_{st}(\lambda)$, then ($S^*, \overline{S}^*)$ is the optimal solution to HNC+. The proof provided in \citep{hochbaum2009polynomial} is omitted here. 
Accordingly, we predict that unlabeled samples in $S^*$ are positive, and those in $\overline{S}^*$ negative.

HNC- may also be used for binary classification and can be solved similarly, via (\ref{eq:HNC-lambda-final-2}) for a tradeoff parameter $\lambda=\frac{\mu}{\mu+2}$, as a minimum cut problem on the associated graph, $G^{-}_{st}(\lambda)$, illustrated in Figure \ref{fig:HNC-2}. The only difference between $G^{+}_{st}(\lambda)$ and $G^{-}_{st}(\lambda)$ is that, in the latter, each $i \in U$ is connected to $t$, rather than $s$, with capacity of $\lambda d_i$.

\begin{figure}[h!]
\centering
\begin{subfigure}{.49\textwidth}
  \centering
  \includegraphics[width=.6\linewidth]{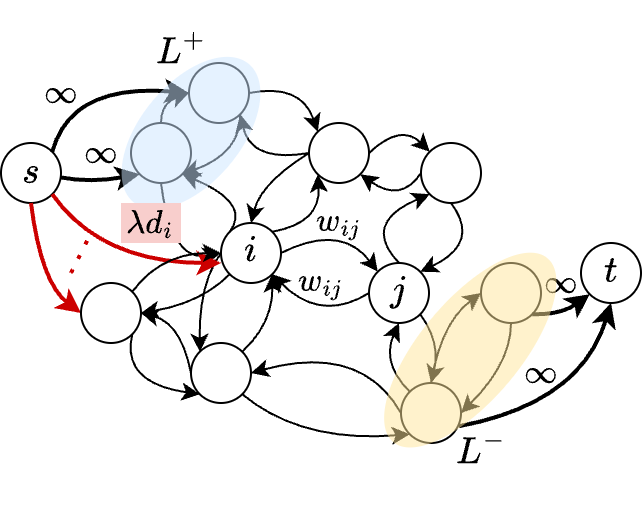}
  \caption{Graph $G_{st}^{+}(\lambda)$ for solving HNC+ with \\ the constraint $L^+ \subseteq S \subseteq V \backslash L^-$. 
  } \label{fig:HNC-1}
\end{subfigure}%
\begin{subfigure}{.49\textwidth}
  \centering
  \includegraphics[width=.6\linewidth]{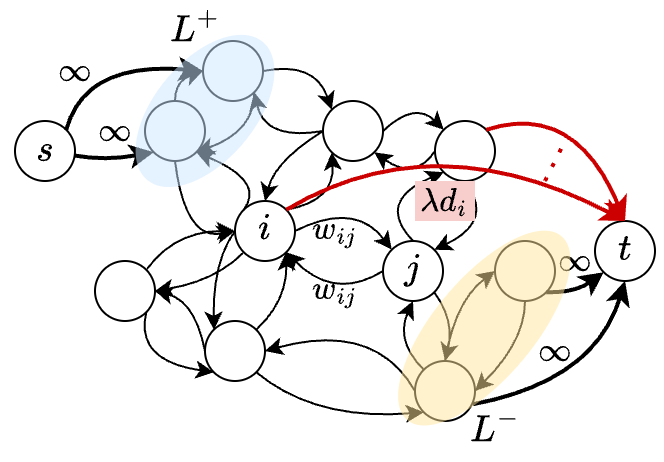}
  \caption{Graph $G_{st}^{-}(\lambda)$ for solving HNC- with \\ the constraint $L^+ \subseteq S \subseteq V \backslash L^-$.} \label{fig:HNC-2}
\end{subfigure}
\caption{Associated graphs with HNC+ and HNC- formulations,
when labeled samples from both classes are given. Nodes in the middle, outside the blue and yellow shaded areas, correspond to unlabeled samples in $U = V \backslash (L^+ \cup L^-)$.}
\label{fig:HNC-graph}
\end{figure}

In PU learning, negative labeled samples are not given. Therefore, $L^- = \varnothing$. HNC+ and HNC- in this context are then solved, for a tradeoff parameter $\lambda$,
as minimum cut problems on the graphs in Figure \ref{fig:HNC-PU-1} and \ref{fig:HNC-PU-2}, which are $G_{st}^{+}(\lambda)$ and $G_{st}^{-}(\lambda)$ where $L^- = \varnothing$. 
As explained in Section \ref{subsec:2HNC-Stage1}, HNC+, with $L^- = \varnothing$, has a trivial solution for all $\lambda\geq 0$. This is also reflected in the minimum cut of $G_{st}^{+}(\lambda)$ (Figure \ref{fig:HNC-PU-1}) with $L^- = \varnothing$, which is $(\{s\} \cup V, \{t\})$, as $t$ is disconnected from other nodes. Hence, in stage 1, we solve only HNC- using the graph $G_{st}^{-}(\lambda)$ in Figure \ref{fig:HNC-PU-2}. Once the likely-negative samples are used as seed samples in stage 2 (Section \ref{subsec:2HNC-Stage2}), HNC+ can be solved using the graph $G_{st}^{+}(\lambda)$ in Figure \ref{fig:HNC-1}.

\begin{figure}[h!]
\centering
\begin{subfigure}{.49\textwidth}
  \centering
  \includegraphics[width=.6\linewidth]{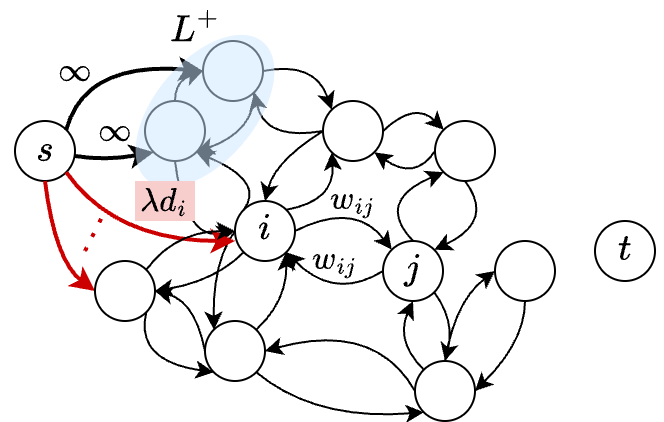}
  \caption{Graph $G_{st}^{+}(\lambda)$ for solving HNC+ with \\ the constraint $L^+ \subseteq S$, when $L^{-}=\varnothing$.} \label{fig:HNC-PU-1}
\end{subfigure}
\begin{subfigure}{.49\textwidth}
  \centering
  \includegraphics[width=.6\linewidth]{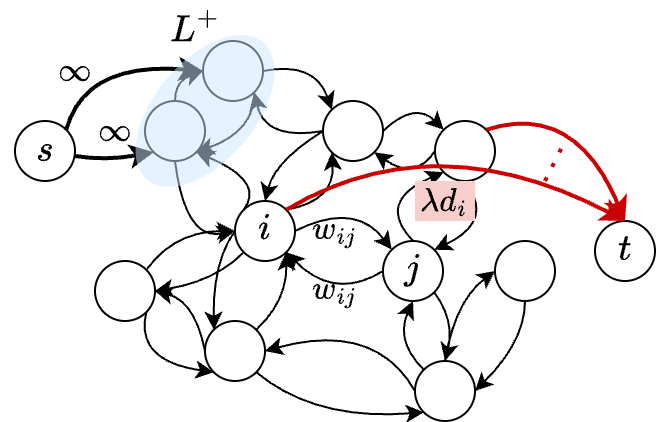}
  \caption{Graph $G_{st}^{-}(\lambda)$ for solving HNC- with \\ the constraint $L^+ \subseteq S$, when $L^{-}=\varnothing$.} \label{fig:HNC-PU-2}
\end{subfigure}
\caption{Graphs on which we solve HNC+ and HNC- as minimum cut problems, in PU learning where negative labeled samples are not provided.}
\label{fig:HNC-PU-graph}
\end{figure}

\subsubsection{Solving HNC+ and HNC- for All Tradeoff Values in a Given List with a Parametric Cut Procedure} \label{subsec:hnc-param}
Graphs $G_{st}^{+}(\lambda)$ and $G_{st}^{-}(\lambda)$ are {\em parametric flow networks}
in that the capacities of source-adjacent and sink-adjacent arcs ($(s,i)$ and $(i,t)$ for $i\in V$) are monotone non-increasing and non-decreasing with the parameter $\lambda$, or vice versa.
For instance, $G_{st}^{+}(\lambda)$ in Figure \ref{fig:HNC-1}, has source-adjacent capacities that can only increase with $\lambda$, and sink-adjacent capacities that are fixed.
Note that both graphs are also parametric flow graphs with the parameter $\mu$ since $\lambda = 1-\frac{2}{\mu+2}$ is an increasing function with respect to $\mu$. 
The minimum cuts in a parametric flow network are solved for all values of the parameter in the complexity of a single minimum cut procedure using the parametric cut (flow) algorithm, \citep{gallo1989fast,hochbaum1998pseudoflow,hochbaum2008pseudoflow}.
The first is based on the push-relabel algorithm, and the latter two on the HPF (pseudoflow) algorithm.

For our method, 2-HNC, HNC- with no negative seed for $\overline{S}$ ($L^{-}=\varnothing$) and HNC+ with the seed set $L^{-}=L^N$ for $\overline{S}$ are solved for a list of nonnegative values of the tradeoff parameter $\lambda$ in stage 1 and 2, respectively, with the simple parametric cut procedure. 

In stage 1, HNC- with $L^{-}=\varnothing$ is solved on the graph $G_{st}^{-}(\lambda)$ in Figure \ref{fig:HNC-PU-2}. 
As explained in Section \ref{subsec:2HNC-Stage1}, the result is a sequence of minimum cuts, or data partitions, for increasing values of $\mu$ (and also $\lambda$), with the nested cut property since $G_{st}^{-}(\lambda)$, in Figure \ref{fig:HNC-PU-2}, is a parametric flow graph.

This nested data partitions sequence, that is the output of stage 1, is then used in stage 2 (Section \ref{subsec:2HNC-Stage2}) to find the likely-negative set, $L^N$. HNC+ with $L^- = L^N$ is solved on the graph $G_{st}^{+}(\lambda)$ in Figure \ref{fig:HNC-1}.

At the end of stage 2, we obtain the predictions of unlabeled samples by selecting one partition, from all partitions that are generated in stage 1 and 2, whose positive fraction is closest to the prior $\pi$. 

\subsubsection{Sparsification} \label{subsec:sparsification}
A general drawback of using minimum cut in dense graphs is that the solution tends to not favor ``balanced" partitions.  In a balanced partition, there is a constant fraction $f<1$ of nodes on one side, and the number of edges between the two sets in the partition is $fn(1-f)n$, which is $O(n^2)$.  In that case, even if many edges in the partition have small capacities, their sheer number makes the capacity of such cuts much higher than cuts that contain a small number of nodes on one side.  
In the graphs we study, all pairwise similarities are evaluated.  Therefore, such graphs are complete and dense. One approach for addressing this issue is to apply {\bf graph sparsification}. 
The sparsification technique that we found to be effective for 2-HNC is the $k$-nearest neighbor (kNN) sparsification \citep{blum2001learning} where samples $i$ and $j$ are connected only if $i$ is among the $k$ nearest neighbors of $j$, or vice versa. 
This results in a graph representation $G=(V,E)$ where $E$ is the set of pairs of samples where in each pair, one sample is one of the $k$ nearest neighbors of the other sample in the pair. 

\section{Implementation of 2-HNC}\label{sec:implementation}

This section includes the specification of several implementation details. First, we give a brief description of the parametric cut solver used in this work. Second, we describe the choice of $k$ in the k-nearest neighbor graph sparsification method, as mentioned in the previous section. Finally, we quantify the pairwise similarity measure between pairs of samples. Our code implementation of 2-HNC can be found at \href{https://github.com/hochbaumGroup/Positive-Unlabeled-Learning}{https://github.com/hochbaumGroup/Positive-Unlabeled-Learning}.

\subsection{Parametric Cut Solver}

Solving HNC, via (\ref{eq:HNC-lambda-final}) and (\ref{eq:HNC-lambda-final-2}), for all nonnegative tradeoff $\lambda$ as a parametric cut problem on an associated parametric flow graph can be done using the fully parametric pseudoflow algorithm, described in \citep{hochbaum2008pseudoflow}, with the available implementation at \citep{fullpara}. 

As discussed in Section \ref{subsec:nestedcut}, we run the experiments using the simple implementation, which is available at \citep{simplepara}, since it runs faster in practice than the fully implementation. We apply the solver twice, once on the graph $G_{st}^{+}(\lambda)$ to solve HNC+ (\ref{eq:HNC-lambda-final}) and the other on $G_{st}^{-}(\lambda)$ to solve HNC- (\ref{eq:HNC-lambda-final-2}). The outputs are the sequences of cut partitions of the two parametric flow graphs, for $\lambda$ in the given list.

The values of $\lambda$ that we use for both $G_{st}^{+}(\lambda)$ and $G_{st}^{-}$ are evenly spaced increments of 0.001, starting from 0 and ending at 0.500, yielding a total of 501 values: ${ 0, 0.001, 0.002, 0.003, \dots, 0.500 }$. We select this range of $\lambda$ since it allows the sink set $\overline{S}$ to expand to $V \backslash L^{+}$, which is the largest set possible, in stage 1 and the source set to grow to $V \backslash L^N$, also largest possible, in stage 2 for all datasets in our experiments. 
The choice of the increment of $0.001$ is used since smaller increments do not produce additional partitions that are different from those generated by using the increment of $0.001$; hence, it is small enough for datasets used in the experiments. 
For other datasets that are not included here, the appropriate range and increment of values of $\lambda$ may vary depending on characteristics of the datasets such as the size and scale of feature values.

\subsection{Graph Sparsification} \label{subsec:sparsify}

As described in Section \ref{subsec:sparsification}, we apply the kNN sparsification to $G_{st}^{+}(\lambda)$ and $G_{st}^{-}(\lambda)$ on which we solve the parametric cut problem.
For each dataset of size smaller than $10000$, we use several values of $k$, $k \in \{5,10,15\}$, and find the partitions for all of them prior to selecting one for the prediction. For larger datasets we use $k = 5$.

The procedure to select a partition from those generated by all $k$'s, for small datasets, is as follows: For each $k$, we find the parametric cut on the kNN-sparsified graph and select the partition whose positive fraction is closest to $\pi$ as the \textit{candidate} partition. Among the candidate partitions from all $k$, we choose the one with the largest $k$ that has its positive fraction within $2\%$ from $\pi$. Larger $k$ is preferred since it maintains more pairwise information. If no candidate partition has positive fraction within $2\%$ from $\pi$, we choose the one with the fraction closest to $\pi$.

Here, only the smallest value of $k$, which is $5$, is used on large datasets because, as discussed in Section \ref{subsec:sparsification}, large datasets, with dense graph representation, often have highly unbalanced cuts. These large datasets benefit from a higher degree of sparsification.

In our work, we use the implementation of the $k$-nearest neighbors search from Scikit-learn \citep{scikit-learn} that relies on the k-d tree data structure and has a runtime complexity of $O(N\log{N})$ where $N$ is the number of samples in the dataset.

\subsection{Pairwise Similarity Computation} \label{subsec:similarity}
Given $H$-dimensional vector representations of samples $i$ and $j$, $\bm{x}_i, \bm{x}_j \in \mathbb{R}^H$, we compute their distance $d_{ij}$ as a weighted Euclidean distance between $\bm{x}_i$ and $\bm{x}_j$. The pairwise similarity $w_{ij}$ is then computed using the Gaussian kernel, which is commonly used in methods that rely on pairwise similarities \citep{jebara2009graph,de2013influence,baumann2019comparative}, as $w_{ij} = exp(-d_{ij}^2/2\sigma^2)$. We use $\sigma = 0.75$ for datasets with less than $10000$ samples. For larger datasets, we use $\sigma=0.25$. Again, large datasets require a higher degree of graph sparsification. Hence, a smaller $\sigma$ is applied so that similarities of distant pairs are brought closer to zero, for the same effect as the sparsification technique discussed in Section \ref{subsec:hnc-param} and \ref{subsec:sparsify}.

The weighted Euclidean distance is computed as $d_{ij} = \sqrt{\sum_{h=1}^H \rho_h (\bm{x}_{ih}-\bm{x}_{jh})^2}$ where $\rho = [\rho_1, \dots, \rho_H]$ is the weight for the feature vector of size $H$. $\rho$ is scaled so that $\sum_{h=1}^H \rho_h = H$, as we want the weights to be comparable to the standard Euclidean distance, in which $\rho = [1, \dots, 1]$. We use \textit{the feature importance} from a random forest-based PU learning method \citep{wilton2022positive} as the weight $\rho$. Features with high importance contribute to high impurity reduction at tree node splits in the random forest. 

Since we apply the kNN sparsification, we only need to compute the pairwise similarities of $O(N)$ pairs, where $N$ is the number of samples in the dataset.

\section{Experiments} \label{sec:experiment}

We compare the performance of 2-HNC with that of benchmark methods on both synthetic datasets and real datasets.

\subsection{Datasets}

For both synthetic and real datasets, we divide each dataset into the labeled set and the unlabeled set as follows. The labeled set $L^+$ is made up of $60\%$ of positive samples, selected randomly. The unlabeled set $U$ consists of the remaining $40\%$ of the positive samples and the entire set of negative samples. The models are evaluated on their labels' predictions of samples in $U$, using evaluation metrics provided in Section \ref{subsec:metrics}.

\subsubsection{Synthetic Datasets}

We generate synthetic datasets using different configurations as summarized in Table \ref{table:synthetic}. The parameters in the configuration  of dataset generation that we vary in our experiments include the number of samples and features,  the percentage of positive samples or $\pi$, the number of clusters of data points per each of the two classes, positions of the clusters (whether they are put on the vertices of a hypercube or a random polytope), and the class separation factor, which indicates distances between the vertices of a hypercube or a polytope. Reported in the table are the values of these parameters that are used. For more detailed information regarding the generation of datasets, see the documentation of the scikit-learn library \citep{scikit-learn}.

\begin{table*}[htb!]
    \caption{Configurations for the generation of synthetic datasets} \label{table:synthetic}
    \centering
    \small
    \begin{tabular}{ll}
    \toprule
    Parameters of data generation & Values \\ \midrule
    \# samples & 1000, 5000, 10000\\
    \# features & 5, 10, 20 \\
    \% positive samples ($\pi$) & 30, 40, 50, 60, 70 \\
    \# clusters per class & 2, 4\\
    class sep. & 0.5, 1.0, 2.0 \\
    hypercube & true, false \\
    \bottomrule
    \end{tabular}
\end{table*}

There are $3 \times 3 \times 5 \times 2 \times 2 \times 3 = 540$ combination of different values across all configuration parameters, resulting in $540$ configurations. For each configuration, we generate $4$ different datasets. Hence, in total, $2160$ synthetic datasets are included in the experiments. 

\subsubsection{Real Datasets}

Datasets are listed in Table \ref{table:datasets}, with the number of all samples, labeled and unlabeled samples, the number of features and the fraction of positive samples ($\pi$) of each dataset. 

\begin{table*}[htb!]
  \caption{Datasets: 60\% of positive samples are randomly selected as labeled samples. The unlabeled set consists of negative samples and the remaining 40\% of positive samples. $\pi$ is the fraction of positive samples in each dataset.}
  \label{table:datasets}
  \centering
  \small
  \begin{tabular}{llllll}
    \toprule
    Name & \# Samples &\# Labeled & \# Unlabeled & \#
    Features & $\pi$ \\ \midrule
    Vote & 435 & 160 & 275 & 16 & 0.61\\
    German & 1000 & 420 & 580 & 24 & 0.70 \\
    Obesity & 2111 & 583 & 1528 & 19 & 0.46 \\
    Mushroom & 8124 & 2524 & 5600 & 112 & 0.52\\
    Phishing & 11055 & 2938 & 8117 & 30 & 0.44 \\
    Bean & 13611 & 3709 & 9902 & 16 & 0.45 \\
    News & 18846 & 6386 & 12460 & 300 & 0.56\\
    Letter & 20000 & 5964 & 14036 & 16 & 0.50\\
    Credit & 30000 & 14018 & 15982 & 23 & 0.60\\
    MNIST & 70000 & 21349 & 48651 & 784 & 0.51\\
    \bottomrule
  \end{tabular}
\end{table*}

All datasets are from the UCI ML Repository \citep{Kelly_Longjohn_Nottingham}, except for 
\textit{News} and \textit{MNIST} that are obtained from scikit-learn \citep{scikit-learn}. 
Samples in each dataset are assigned labels (positive vs negative) as follow: \textit{Vote}: \{Democrat\} vs \{Republican\}, \textit{German}: \{Good credit risk\} vs \{Bad credit risk\}, \textit{Obesity}: \{Obesity Type I, II and III\} vs \{Insufficient, Normal, Overweight\}, \textit{Mushroom}: \{Edible\} vs \{Poisonous\}, \textit{Phishing Websites}: \{phishing\} vs \{not phishing\}, \textit{Bean}: \{Sira, Dermason\} vs \{Seker, Barbunya, Bombay, Cali, Horoz\}, \textit{News}: \{alt., comp., misc., rec.\} vs \{sci., soc., talk.\}, \textit{Letter}: \{A-M\} vs \{N-Z\}, \textit{Credit}: \{paid\} vs \{default\}, \textit{MNIST}: \{1,3,5,7,9\} vs \{0,2,4,6,8\}. 
Following the works of \citet{kiryo2017positive} and \citet{wilton2022positive}, we use a pre-trained GloVe word embedding \citep{pennington2014glove} to map each document in \textit{News} to a 300-dimension vector.

\subsection{Benchmark Methods}
2-HNC is compared against the following benchmarks: uPU \citep{du2014analysis,du2015convex}, nnPU \citep{kiryo2017positive} and \text{PU ET} \citep{wilton2022positive}, as discussed in Section \ref{sec:related-works}. 

The neural networks we use for uPU and nnPU are the same as previously used in \citep{kiryo2017positive}: a 6-layer MLP with Softsign activation function for \textit{News}, a 13-layer CNN with a ReLU final layer for \textit{MNIST}, and a 6-layer MLP with ReLU for other datasets. We use the available implementations \citep{code_upu_nnpu} of uPU and nnPU. For PU ET, we use the available implementation \citep{code_puet} and use the default hyperparameters as suggested in \citep{wilton2022positive}.

\subsection{Evaluation Metrics} \label{subsec:metrics}

We evaluate 2-HNC and the benchmark methods using two metrics on the label predictions of unlabeled samples: accuracy and balanced accuracy. Accuracy is the fraction of correct predictions among all predictions;  balanced accuracy is the average of the accuracy among the unlabeled samples from the positive class and the accuracy of those from the negative class.  Formally, denoting the counts of true positives, true negatives, false positives and false negatives by $TP, TN, FP$ and $FN$, these metrics are defined as $Accuracy = \frac{TP+TN}{TP+FP+TN+FN}$ and $Balanced \; Accuracy = \frac{1}{2}(\frac{TP}{TP+FN} + \frac{TN}{TN+FP})$.

\subsection{Results on Synthetic Datasets} 
We measure the accuracy and balanced accuracy given by 2-HNC as well as benchmark methods on each synthetic dataset. 
The overall results across all synthetic datasets are summarized via histograms. The histograms of accuracy and balanced accuracy improvement yielded by 2-HNC over benchmark methods are shown in Figures \ref{fig:syn-acc} and \ref{fig:syn-balacc}. 
Suppose the accuracy (or balanced accuracy) of 2-HNC on dataset $D$ is $a_{2HNC}(D)$ and that of a benchmark method $B$ on the same dataset is $a_B(D)$. The accuracy (or balanced accuracy) improvement of 2-HNC over $B$ on $D$ is then calculated as $(\frac{a_{2HNC}(D)}{a_{B}(D)} -1)\times 100\%$.
On the histograms, the vertical dashed line marks the average improvement, as percentage, yielded by 2-HNC over each benchmark method, whereas the vertical solid line marks the zero improvement, which is where 2-HNC and the respective benchmark for the plot have equal performance. 

\begin{figure}
  \centering
  \begin{subfigure}{.33\textwidth}
    \centering
    \includegraphics[width=\linewidth]{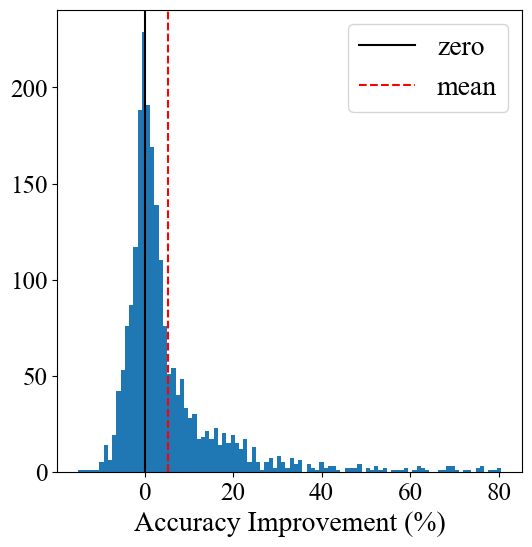}
    \caption{2-HNC vs uPU}
    \label{fig:uPU-syn-acc}
  \end{subfigure}%
  \begin{subfigure}{.33\textwidth}
    \centering
    \includegraphics[width=\linewidth]{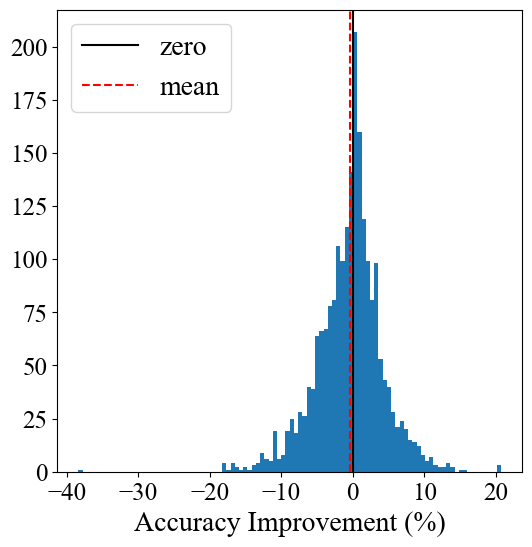}
    \caption{2-HNC vs nnPU}
    \label{fig:nnPU-syn-acc}
  \end{subfigure}
  \begin{subfigure}{.33\textwidth}
    \centering
    \includegraphics[width=\linewidth]{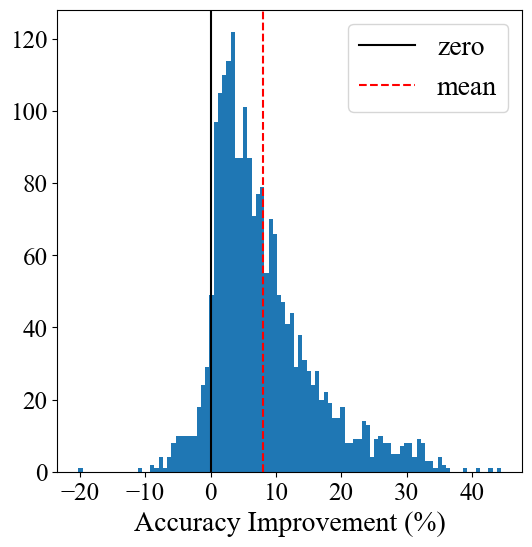}
    \caption{2-HNC vs PUET}
    \label{fig:PUET-syn-acc}
  \end{subfigure}
  \caption{Histograms of accuracy improvement yielded by 2-HNC over (a) uPU, (b) nnPU, and (c) PUET,  on 2160 synthetic datasets.}
  \label{fig:syn-acc}
\end{figure}

\begin{figure}
  \centering
  \begin{subfigure}{.33\textwidth}
    \centering
    \includegraphics[width=\linewidth]{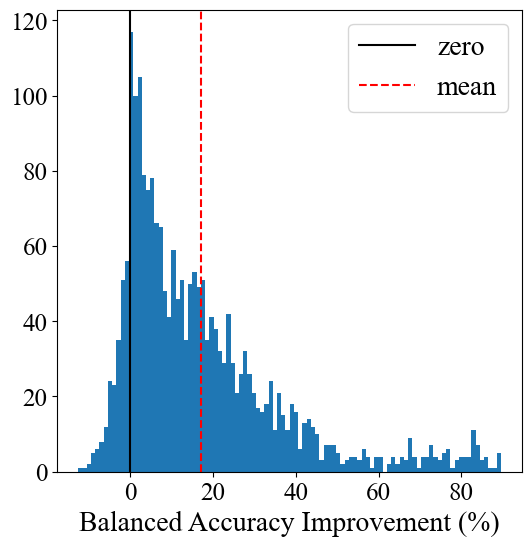}
    \caption{2-HNC vs uPU}
    \label{fig:uPU-syn-balacc}
  \end{subfigure}%
  \begin{subfigure}{.33\textwidth}
    \centering
    \includegraphics[width=\linewidth]{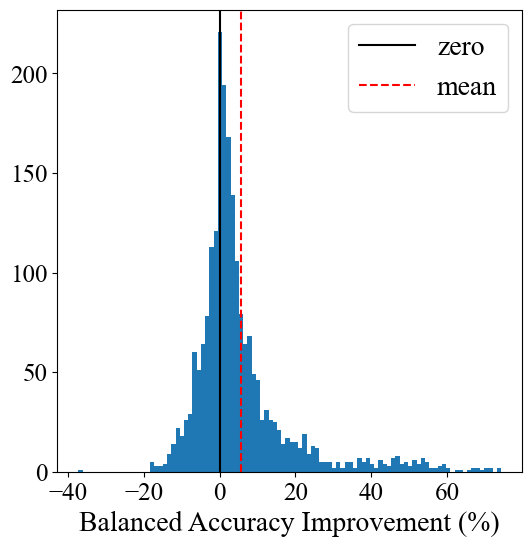}
    \caption{2-HNC vs nnPU}
    \label{fig:nnPU-syn-balacc}
  \end{subfigure}
  \begin{subfigure}{.33\textwidth}
    \centering
    \includegraphics[width=\linewidth]{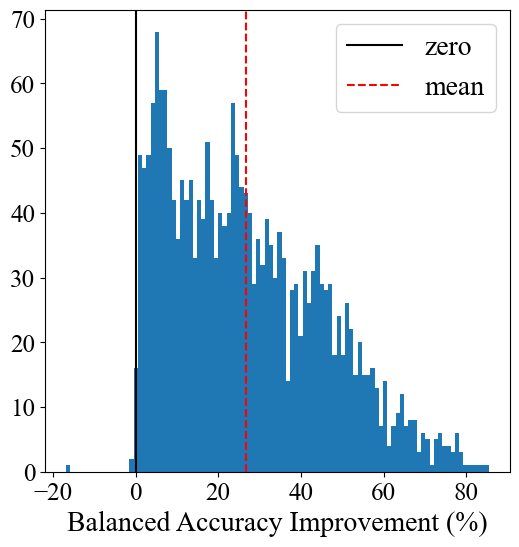}
    \caption{2-HNC vs PUET}
    \label{fig:PUET-syn-balacc}
  \end{subfigure}
  \caption{Histograms of balanced accuracy improvement yielded by 2-HNC over: (a) uPU, (b) nnPU, and (c) PUET, on 2160 synthetic datasets.}
  \label{fig:syn-balacc}
\end{figure}

The histogram plots in Figures \ref{fig:uPU-syn-acc} and \ref{fig:PUET-syn-acc}
show that, on average, 2-HNC gives positive improvement in terms of classification accuracy over uPU and PU ET, as the red dashed lines (the mean) lie to the right of the black solid line (the zero line). The average improvements, reported in Table \ref{table:syn-acc}, show the improvements of 2-HNC over uPU and PUET by $5.33\%$ and $8.02\%$, respectively. The fractions of datasets where 2-HNC improves on these two benchmarks are more than half, and in particular, as compared to PU ET, 2-HNC produces better results on almost all datasets ($92.96 \%$). 
However, 2-HNC does not deliver consistent improvement in accuracy over nnPU, with the average improvement that is slightly below zero, at $-0.44\%$, reported in Table \ref{table:syn-acc} and shown in Figure \ref{fig:PUET-syn-acc} where the red dashed line (the mean) lies almost on the solid line (the zero line). The percentage of datasets where 2-HNC improves over nnPU is just above $50 \%$.

For balanced accuracy, 2-HNC's performance improves over all benchmark methods, with the average improvement of $16.96\%, 5.57 \%$ and $26.75\%$ over uPU, nnPU and PU ET, respectively, as shown via histograms in \ref{fig:uPU-syn-balacc}, \ref{fig:nnPU-syn-balacc} and \ref{fig:PUET-syn-balacc} and reported in Table \ref{table:syn-balacc}. The percentage of datasets where 2-HNC improves over these three benchmarks are $87.87 \%, 66.94 \%$ and $99.86 \%$, respectively.

To explain the difference in performance between the two evaluation metrics, we note that there are synthetic datasets that are highly unbalanced. For these datasets, the three benchmark methods tend to assign all samples to the majority class to achieve high accuracy. 
By doing so, these methods obtain low accuracy on the minority class, and hence, low balanced accuracy when averaging the accuracies of the two classes. In contrast, 2-HNC assigns samples to both classes since the method selects the partition with the positive fraction closest to the prior $\pi$. Hence, it can obtain sufficiently high accuracy within each class, resulting in better balanced accuracy despite the lower accuracy in the majority class.

\begin{table*}[h!]
  \caption{Summary of accuracy improvement yielded by 2-HNC over benchmark methods. Reported numbers are the average improvement (\%), with standard deviation, across all synthetic datasets and the percentage of datasets where 2-HNC shows positive improvement.} \label{table:syn-acc}
  \centering
  \small
  \begin{tabular}{llll}
    \toprule
     & 2-HNC vs uPU & 2-HNC vs nnPU & 2-HNC vs PU ET \\ \midrule
    Avg. Imp. (\%) & 5.33 (12.33) & -0.44 (4.68) & 8.02 (7.90)\\
    \% Data with Imp. (\%) & 65.05 & 50.19 & 92.96 \\
    \bottomrule
  \end{tabular}
\end{table*}

\begin{table*}[h!]
  \caption{Summary of balanced accuracy improvement yielded by 2-HNC over benchmark methods. Reported numbers are the average improvement (\%), with standard deviation, across all synthetic datasets and the percentage of datasets where 2-HNC shows positive improvement.} \label{table:syn-balacc}
  \centering
  \small
  \begin{tabular}{llll}
    \toprule
     & 2-HNC vs uPU & 2-HNC vs nnPU & 2-HNC vs PU ET \\ \midrule
    Avg. Imp. (\%) & 16.96 (19.54) & 5.57 (13.45) & 26.75 (18.70)\\
    \% Data with Imp. (\%) & 87.87 & 66.94 & 99.86 \\
    \bottomrule
  \end{tabular}
\end{table*}

\subsection{Results on Real Datasets} 

The accuracy and balanced accuracy of 2-HNC and benchmark models, for real datasets, are summarized in Table \ref{table:acc} and \ref{table:balacc}. We report the average and standard error of the scores across 5 runs of each method on each dataset, where each run corresponds to one randomly generated split of dataset into labeled and unlabeled samples. For datasets where 2-HNC is dominated by at least one of the benchmark methods, we report the percentage difference between the highest accuracy (or balanced accuracy) and the accuracy given by 2-HNC, as shown in the last column of the tables. 

In terms of classification accuracy, 2-HNC yields the best result on all datasets except for \textit{German}, \textit{Phishing} and \textit{MNIST}. For \textit{German} and \textit{Phishing}, 2-HNC achieves the second highest accuracy, only $0.41\%$ and $0.83\%$ below the highest achieved accuracy, which are given by nnPU and PU ET, respectively. For \textit{MNIST}, the accuracy given by 2-HNC is $0.67\%$ below the highest accuracy given by PU ET.

Regarding balanced accuracy, 2-HNC gives the best result on all datasets, except for \textit{German} and \textit{MNIST}, where 2-HNC attains the second highest balanced accuracy, that is only $0.52\%$ and $0.20\%$ below that attained by nnPU, for \textit{German} and \textit{MNIST}, respectively.

Overall, 2-HNC is the best performer for the real datasets with respect to both evaluation metrics. Among the three benchmark models, the most competitive method is PU ET, for accuracy, and nnPU, for balanced accuracy.

\begin{table*}[h!]
  \caption{Classification accuracy (\%) average (and standard error) across 5 runs of 2-HNC and benchmark methods. For each dataset (row), the number with asterisk (*) is the highest accuracy and numbers in bold are the highest two.}
  \label{table:acc}
  \centering
  \small
  \begin{tabular}{llllll}
    \toprule
    Dataset & uPU & nnPU & PU ET & 2-HNC & Diff Max \\ \midrule
    Vote & 69.38 
    (1.50) & 93.38 (1.65) & \textbf{93.60} (1.09) & \textbf{96.15*} (0.54) & \\
    German & 61.31 (5.38) & \textbf{68.83*} (4.25) & 65.34 (0.65) & \textbf{68.55} (1.24) & -0.41\%\\
    Obesity & 78.98 (21.46) & 89.28 (7.44) & \textbf{96.36} (0.62) & \textbf{97.81*} (0.44) & \\
    Mushroom & 99.41 (0.45) & 99.61 (0.35) & \textbf{99.67} (0.18) & \textbf{99.83*} (0.24) & \\
    Phishing  & 93.78 (0.51) & 94.64 (0.23) & \textbf{95.75*} (0.27) & \textbf{94.96} (0.29) & -0.83\%\\
    Bean  & 68.91 (5.53) & 82.35 (9.89) & \textbf{97.27} (0.07) & \textbf{97.34*} (0.06) & \\
    News & 80.06 (9.29) & \textbf{89.77} (2.89) & 84.8 (0.19) & \textbf{91.35*} (0.13) & \\
    Letter &\textbf{93.88} (0.39) & 93.62 (0.85) & 91.93 (0.23) & \textbf{97.5*} (0.18) & \\
    Credit & 64.08 (4.11) & 64.08 (4.11) & \textbf{67.23} (0.24) & \textbf{69.47*} (0.5) & \\
    MNIST & \textbf{96.91} (0.2) & 96.51 (0.44) & \textbf{97.16*} (0.1) & 96.77 (0.1) & -0.67\%\\
    \bottomrule
  \end{tabular}
\end{table*}

\begin{table*}[h!]
  \caption{Balanced accuracy (\%) average (and standard error) across 5 runs of 2-HNC and benchmark methods. For each dataset (row), the number with asterisk (*) is the highest balanced accuracy and numbers in bold are the highest two.}
  \label{table:balacc}
  \centering
  \small
  \begin{tabular}{llllll}
    \toprule
    Dataset & uPU & nnPU & PU ET & 2-HNC & Diff Max \\ \midrule
    Vote & 60.65 (39.35) & \textbf{92.34} (4.68) & 92.22 (6.24) & \textbf{95.9*} (1.13) & \\
    German & 60.63 (19.63) & \textbf{68.89*} (1.75) & 64.6 (21.53) & \textbf{68.53} (0.67) & -0.52\%\\
    Obesity & 75.17 (7.76) & 89.17 (0.22) & \textbf{92.87} (7.11) &  \textbf{97.2*} (1.26) & \\
    Mushroom & 99.02 (0.98) & \textbf{99.57} (0.09) & 99.45 (0.55) &  \textbf{99.78*} (0.14) & \\
    Phishing & 87.81 (11.54) & 91.89 (5.33) & \textbf{92.84} (5.63) & \textbf{93.18*} (3.44) & \\
    Bean & 55.99 (25.82) & 74.2 (16.29) & \textbf{96.01} (2.52)& \textbf{96.44*} (1.8) & \\
    News & 71.36 (27.47) & \textbf{87.99} (5.63) & 79.02 (18.26) &  \textbf{90.38*} (3.05) & \\
    Letter & 90.19 (8.49) & \textbf{91.51} (4.87) & 85.85 (14.03) &   \textbf{96.95*} (1.28) & \\
    Credit & 59.39 (27.69) & 59.39 (27.69) & \textbf{67.3} (0.36) &  \textbf{68.54*} (5.47) & \\
    MNIST & 95.7 (2.9) & \textbf{96.27*} (0.58) & 95.78 (3.34) & \textbf{96.08} (1.65) & -0.20\%\\
    \bottomrule
  \end{tabular}
\end{table*}

We next report on the {\em statistical significance} of (i) the outperformance of 2-HNC over the second best method on datasets where 2-HNC achieves the highest performance and (ii) the underperformance of 2-HNC below the best benchmark on datasets where 2-HNC is not the best performing method.  Two statistical tests are used: the paired t-test and the Wilcoxon signed-rank test. The paired t-test is a common statistical test that makes several assumptions on the underlying numbers that are compared, such as the normality of the differences between the accuracies of the two methods \citep{demvsar2006statistical}. Wilcoxon test is an appropriate and common alternative when one compares two classifiers over multiple datasets, without relying on the assumption that the differences are normally distributed \citep{demvsar2006statistical}. 

The results of the paired t-test are reported in Table \ref{table:pvalues-outperform} and \ref{table:pvalues-underperform} for the outperformance and underperformance, respectively. The results of the Wilcoxon signed-rank test are consistent with those of the paired t-test and are described in text in the following paragraphs.

In Table \ref{table:pvalues-outperform}, we only report p-values for datasets and evaluation metrics where 2-HNC has the best performance. Hence, \textit{German} and \textit{MNIST} are not included here, as well as the p-value for the accuracy on \textit{Phishing}, where 2-HNC is dominated by PU ET. The alternative hypothesis for the test is that 2-HNC outperforms the best benchmark. 

We see that in most cases, the t-test p-values are below $0.05$, indicating that we reject the null hypothesis and that 2-HNC outperforms the best benchmark with high statistical significance ($\alpha=0.05$). The exceptions are \textit{Mushroom} and \textit{News} data where 2-HNC has the best performance on average but without high statistical significance. The Wilcoxon p-values, which are not reported in the table, are consistent with the t-test results. For the pairs of dataset and metric on which t-test p-values show high statistical significance (those with asterisk in Table \ref{table:pvalues-outperform}), their corresponding p-values of Wilcoxon test are $0.0313$, which also demonstrates high statistical significance ($\alpha = 0.05$). 

The Wilcoxon p-values for the outperformance of 2-HNC on \textit{Mushroom}, for accuracy and balanced accuracy, are $0.2188$ and $0.2326$, respectively, and for \textit{News}, they are $0.3125$ and $0.2188$, respectively, all of which are above $0.05$ and do not demonstrate high statistical significance. Overall, 2-HNC outperforms the benchmark methods with high statistical significance on all datasets, except for \textit{Mushroom} and \textit{News}.

\begin{table*} 
  \caption{T-test p-values for the outperformances of 2-HNC over the best benchmarks. (*) denotes p-values where HNC outperforms with high statistical significance ($\alpha=0.05$).}
  \label{table:pvalues-outperform}
  \small
  \centering
  \begin{tabular}{lllll}
    \toprule
    & \multicolumn{2}{c}{Accuracy} & \multicolumn{2}{c}{Balanced Accuracy} \\ \midrule
    \multirow{2}{*}{Dataset} & Best & T-test & Best & T-test \\ 
    & benchmark & p-values & benchmark & p-values\\ \midrule

    Vote & PU ET & 0.0087* & nnPU & 0.0029* \\
    Obesity & PU ET & 0.0009* & PU ET & 0.0007* \\
    Mushroom & PU ET & 0.2306 & nnPU & 0.2610 \\
    Phishing & - & - & PU ET & 0.0103*\\
    Bean & PU ET & 0.0125* & PU ET & 0.0011* \\
    News & nnPU & 0.1704 & nnPU & 0.1340 \\
    Letter & uPU & 1.9e-5* & nnPU & 4.0e-5* \\
    Credit & PU ET & 0.0003* & PU ET & 0.0028* \\
    \bottomrule
  \end{tabular}
\end{table*}
\begin{table*} 
  \caption{T-test p-values for the underperformances of 2-HNC below the best benchmarks. (*) denotes p-values where HNC underperforms with high statistical significance ($\alpha=0.05$).}
  \label{table:pvalues-underperform}
  \small
  \centering
  \begin{tabular}{lllll}
    \toprule
    & \multicolumn{2}{c}{Accuracy} & \multicolumn{2}{c}{Balanced Accuracy} \\ \midrule
    \multirow{2}{*}{Dataset} & Best & T-test & Best & T-test \\ 
    & benchmark & p-values & benchmark & p-values\\ \midrule

    German & nnPU & 0.4451 & nnPU & 0.4302 \\
    Phishing & PU ET & 0.0004* & - & -\\
    MNIST & PU ET & 3.3e-9* & nnPU & 0.0711 \\
    \bottomrule
  \end{tabular}
\end{table*}

Table \ref{table:pvalues-underperform} shows the t-test p-values for the underperformance of 2-HNC below the best benchmarks. For each dataset (row) and each metric (column), the null hypothesis is that 2-HNC performs better than or equally well as the best benchmark. 
The alternative hypothesis is that 2-HNC underperforms the best benchmark. Hence, if the p-value is smaller than $0.05$, we would reject the null hypothesis and conclude that 2-HNC underperforms the best benchmark on that dataset with respect to the corresponding metric, with statistical significance.

Since the p-values on \textit{German} are $0.4451$ and $0.4302$, for the two evaluation metrics, which are much higher than $0.05$, we do not reject the null hypothesis and cannot conclude with high statistical significance that 2-HNC underperforms nnPU on \textit{German}.


This is also the case when we consider the balanced accuracy of 2-HNC and the best benchmark on \textit{MNIST} where the p-value is $0.0711$, which is larger than $0.05$. Hence, it cannot be concluded with statistical significance that 2-HNC has a worse performance than nnPU. 
The underperformance of 2-HNC is statistically significant for its accuracy on \textit{Phishing} and \textit{MNIST} as compared to PU ET, in which the p-values are $0.0004$ and $3.3e\text{-}9$, respectively. 

P-values from the Wilcoxon signed-rank test are consistent with the t-test results. For the comparison of accuracy and balanced accuracy on \textit{German}, and balanced accuracy on \textit{MNIST}, the p-values are $0.5$, $0.5$ and $0.1563$, respectively, which are all above $0.05$ and do not show statistical significance for the superiority of the benchmark method. P-values for the accuracy on \textit{Phishing} and \textit{MNIST} are both $0.0313$, which are below $0.05$ and demonstrate statistical significance for the underperformance of 2-HNC compared to PU ET in this case.

The conclusion from Table \ref{table:pvalues-outperform} and Table \ref{table:pvalues-underperform} is that, for datasets and metrics where 2-HNC achieves better results than all the benchmark methods, it outperforms with high statistical significance, with the exception of the two datasets: \textit{Mushroom} and \textit{News}. 
On the other hand, when 2-HNC does not attain the best result among all models, the underperformance of 2-HNC is statistically significant only for the accuracy metric on \textit{Phishing} and \textit{MNIST}. In other cases, the underperformance of 2-HNC cannot be concluded with high statistical significance. 

\section{Conclusions}

We introduce here a new PU learning method, 2-HNC, which employs pairwise similarities between samples and as such is particularly suitable for the absence of training labels for the negative set, since we leverage additional information from unlabeled samples.
The technique employed makes an innovative use of network flow methods and in particular the parametric minimum cut algorithm.
It is shown here that 2-HNC is competitive as compared to leading techniques, and often improves on their performance. On synthetic datasets, 2-HNC outperforms all benchmark methods in terms of balanced accuracy. It also outperforms uPU and PU ET on accuracy, and exhibits competitive results against nnPU regarding the same metric. On real datasets, 2-HNC achieves the best performance overall for both evaluation metrics, with high statistical significance in most cases.

The use of parametric cut produces a sequence of minimum cut partitions that are used to rank the likelihood that an unlabeled sample is negative.
It is anticipated that this ranking method can be utilized successfully in contexts other than PU learning.

In addition, we introduce here the polynomial solvability of a newly defined clustering problem.  It is shown that 
the problem of trading off the objectives of maximizing intra-similarities in both a cluster and its complement and minimizing the intersimilarity between the cluster and its complement
is solvable as a minimum cut problem on an associated graph.  This is the Double Intra-similarity Theorem, proved here.

\section*{Acknowledgments}
This research was supported by the AI Institute NSF Award 2112533. This research used the Savio computational cluster resource provided by the Berkeley Research Computing program at the University of California, Berkeley (supported by the UC Berkeley Chancellor, Vice Chancellor for Research, and Chief Information Officer).

\bibliographystyle{unsrtnat}  
\bibliography{ref}

\end{document}